%% file: main.tex
\documentclass[10pt,twocolumn,letterpaper]{article}

\usepackage{iccv}
\usepackage{times}
\usepackage{epsfig}

\usepackage{amsmath,amssymb,amsfonts,amsthm,dsfont} 
\usepackage{algorithm,algorithmicx,listings}        
\usepackage{graphicx,tabularx,adjustbox}            

\usepackage{multirow} 
\usepackage{booktabs} 
\usepackage{enumitem} 

\DeclareMathOperator{\tr}{tr}
\DeclareMathOperator{\diag}{diag}
\DeclareMathOperator{\adj}{adj}

\newcommand{\scaleMathLine}[2][1]{\resizebox{#1\linewidth}{!}{$\displaystyle{#2}$}}

\newcommand{\prl}[1]{\left(#1\right)}
\newcommand{\brl}[1]{\left[#1\right]}
\newcommand{\crl}[1]{\left\{#1\right\}}

\def\onedot{.}

\newcommand{\thickhline}{%
    \noalign {\ifnum 0=`}\fi \hrule height 1pt
    \futurelet \reserved@a \@xhline
}
\def\etal{\emph{et al.}}

\theoremstyle{definition}

\newtheorem*{definition*}{Definition}
\newtheorem*{problem*}{Problem}

\newtheorem*{proposition*}{Proposition}
\newtheorem{proposition}{Proposition}

\makeatletter
\newenvironment{proof*}[1][\proofname]{\par
  \pushQED{\qed}%
  \normalfont \partopsep=\z@skip \topsep=\z@skip
  \trivlist
  \item[\hskip\labelsep
        \itshape
    #1\@addpunct{.}]\ignorespaces
}{%
  \popQED\endtrivlist\@endpefalse
}
\makeatother

\input{sym.tex}

\iccvfinalcopy 

\def\iccvPaperID{2741} 

\makeatletter
\def\thanks#1{\protected@xdef\@thanks{\@thanks
        \protect\footnotetext{#1}}}
\makeatother

\ificcvfinal\fi
\begin{document}

\title{ELLIPSDF: Joint Object Pose and Shape Optimization with a Bi-level Ellipsoid and Signed Distance Function Description\thanks{We gratefully acknowledge support from ARL DCIST CRA W911NF-17-2-0181 and NSF RI IIS-2007141.}}

\author{Mo Shan, Qiaojun Feng, You-Yi Jau, Nikolay Atanasov\\
University of California San Diego\\
{\tt\small \{moshan,qjfeng,yjau,natanasov\}@ucsd.edu}
}

\maketitle
\thispagestyle{empty}

\begin{abstract}
Autonomous systems need to understand the semantics and geometry of their surroundings in order to comprehend and safely execute object-level task specifications. 
This paper proposes an expressive yet compact model for joint object pose and shape optimization, 
and an associated optimization algorithm to infer an object-level map from multi-view RGB-D camera observations. 
The model is expressive because it captures the identities, positions, orientations, and shapes of objects in the environment. 
It is compact because it relies on a low-dimensional latent representation of implicit object shape, allowing onboard storage of large multi-category object maps. 
Different from other works that rely on a single object representation format, our approach has a bi-level object model that captures both the coarse level scale as well as the fine level shape details. 
Our approach is evaluated on the large-scale real-world ScanNet dataset and compared against state-of-the-art methods.
\end{abstract}

\input{Introduction.tex}

\input{Review.tex}
\input{Background.tex}

\input{Problem.tex}

\input{Reconstruction.tex}

\input{Evaluation.tex}
\input{Conclusion.tex}

{\small
\bibliographystyle{ieee}
\bibliography{egbib}
}

\newpage 
\input{Supplemental.tex}

\section*{Acknowledgments}
The first author would like to thank Kejie Li at University of Adelaide for helpful discussions.  

\end{document}

%% file: sym.tex
\newcommand{\calE}{{\cal E}}

\newcommand{\calN}{{\cal N}}
\newcommand{\calO}{{\cal O}}

\newcommand{\calS}{{\cal S}}

\newcommand{\calX}{{\cal X}}
\newcommand{\calY}{{\cal Y}}

\newcommand{\bfc}{\mathbf{c}}

\newcommand{\bfe}{\mathbf{e}}

\newcommand{\bfi}{\mathbf{i}}

\newcommand{\bfo}{\mathbf{o}}
\newcommand{\bfp}{\mathbf{p}}
\newcommand{\bfq}{\mathbf{q}}

\newcommand{\bfs}{\mathbf{s}}
\newcommand{\bft}{\mathbf{t}}
\newcommand{\bfu}{\mathbf{u}}
\newcommand{\bfv}{\mathbf{v}}
\newcommand{\bfw}{\mathbf{w}}
\newcommand{\bfx}{\mathbf{x}}
\newcommand{\bfy}{\mathbf{y}}
\newcommand{\bfz}{\mathbf{z}}

\newcommand{\bfepsilon}{\boldsymbol{\epsilon}}

\newcommand{\bftheta}{\boldsymbol{\theta}}

\newcommand{\bfmu}{\boldsymbol{\mu}}

\newcommand{\bfpi}{\boldsymbol{\pi}}
\newcommand{\bfrho}{\boldsymbol{\rho}}
\newcommand{\bfsigma}{\boldsymbol{\sigma}}

\newcommand{\bfphi}{\boldsymbol{\phi}}

\newcommand{\bfxi}{\boldsymbol{\xi}}

\newcommand{\bfA}{\mathbf{A}}

\newcommand{\bfC}{\mathbf{C}}

\newcommand{\bfE}{\mathbf{E}}

\newcommand{\bfH}{\mathbf{H}}
\newcommand{\bfI}{\mathbf{I}}

\newcommand{\bfM}{\mathbf{M}}

\newcommand{\bfP}{\mathbf{P}}
\newcommand{\bfQ}{\mathbf{Q}}
\newcommand{\bfR}{\mathbf{R}}

\newcommand{\bfT}{\mathbf{T}}
\newcommand{\bfU}{\mathbf{U}}
\newcommand{\bfV}{\mathbf{V}}

\newcommand{\bfY}{\mathbf{Y}}

\newcommand{\bbR}{\mathbb{R}}

%% file: Introduction.tex
\section{Introduction}
\label{sec:introduction}

Range sensors, such as RGB-D cameras and LIDARs, have become a primary data source for robot localization and mapping due to their increasing accuracy, affordability, and compactness. This has contributed to the development of RGB-D Simultaneous Localization And Mapping (SLAM) \cite{kerl2013dense, newcombe2011kinectfusion, palazzolo2019refusion, wang2019real} and Structure from Motion (SfM) \cite{agarwal2011building, crandall2011discrete, schonberger2016structure} approaches that provide accurate and efficient ego-motion estimation and map reconstruction. The map representations used in RGB-D SLAM, however, are predominantly geometric, composed of point landmarks \cite{kerl2013dense, sunderhauf2017meaningful}, surfels \cite{whelan2016elasticfusion} or explicit (mesh) and implicit (signed distance field) surface representations \cite{newcombe2011kinectfusion, rosinol2019kimera}. These geometric models do not provide semantic information such as the class, pose, shape, or affordances of objects in the scene. Maps that combine geometric and semantic information are useful and understandable for humans and allow specification of symbolic tasks, such as retrieval, object-directed navigation, grasping, and safety critical operation, in terms of object entities.

Recent works that focus on object-level localization and mapping include~\cite{pronobis2011semantic,salas2013slam++, mu2016slam, sunderhauf2017meaningful, mccormac2018fusion++, hu2019deep}, which utilize objects as landmarks for localization and navigation~\cite{salas2013slam++, Atanasov_SemanticLocalization_IJRR15, mu2016slam, sunderhauf2017meaningful, mccormac2018fusion++, hu2019deep} or as functional entities for motion, part, and affordance identification~\cite{lu2018beyond, qiu2019tracking, mo2019partnet, Luo_2020_ICLR}. The memory and computational efficiency of the object representations used by semantic SLAM are vital for accommodating online construction, onboard storage, and multi-robot use of large semantic maps. On one hand, a parsimonious way for optimizing and storing object maps is needed to ensure online computation and low onboard memory use. On the other hand, it is desirable to preserve as many details about the object shapes, texture, and affordances as possible. Striking the right balance between a faithful object reconstruction and a compact object representation remains an open research problem.

\begin{figure}[t]
  \centering
  \includegraphics[width=1\linewidth]{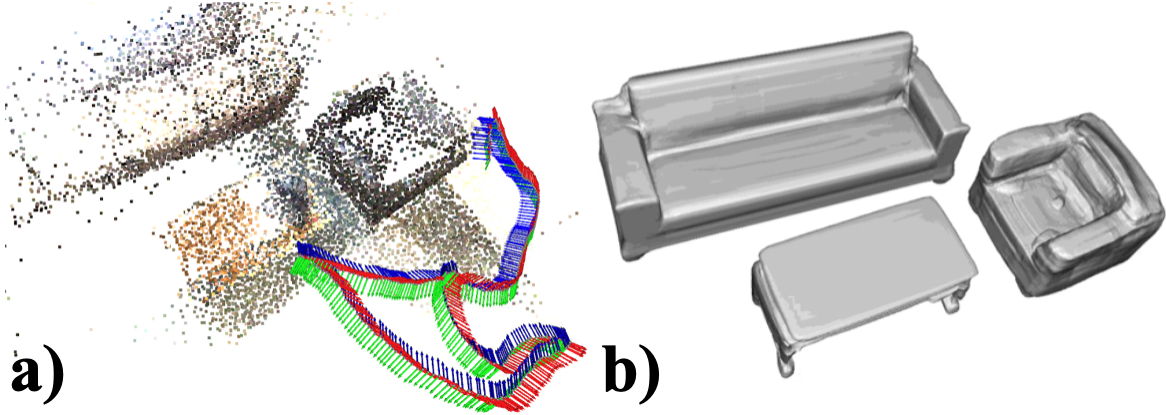}
  \caption{Overview of ELLIPSDF: a) Ground-truth scene reconstruction from colored point clouds in ScanNet~\cite{dai2017scannet} scene $0087$, where the RGB axes show the camera trajectory, b) Reconstructed object meshes in the world frame using the SDF model decoded from a latent code, and the optimized \text{SIM}(3) transformation representing object pose.}
  \label{fig:overview}
\end{figure}

This paper proposes ELLIPSDF, which is an expressive yet compact model of object pose and shape, and an associated optimization algorithm to infer an object-level map from multi-view RGB-D camera observations, as shown in Fig.~\ref{fig:overview}. ELLIPSDF is expressive because it captures the identity, scale, position, orientation, and shape of objects in the environment. It is compact because it relies on a low-dimensional latent encoding of the signed distance function (SDF) to an object's surface, allowing onboard storage of large multi-category object maps.

Shape representation using SDF predicted by an autodecoder network was proposed in DeepSDF~\cite{park2019deepsdf} and DualSDF~\cite{hao2020dualsdf}. In this paper, we extend the SDF prediction network in prior works by proposing a bi-level object model with a shared latent representation. 
Object primitive shapes and SDF are predicted from a shared latent space.
On the coarse-level, an ellipsoid is used as a primitive shape to constrain the overall shape scale. On the fine-level, an autodecoder similar to DeepSDF is used to preserve the object shape details. 
To summarize, the main \textbf{contribution} of this work is the design of
\begin{itemize}
  \item a bi-level object model with coarse and fine levels, enabling joint optimization of object pose and shape. The coarse-level uses a primitive shape for robust pose and scale initialization, and the fine-level uses SDF residual directly to allow accurate shape modeling. The wo levels are coupled via a shared latent space.
  
  \item a cost function to measure the mismatch between the bi-level object model and the segmented RGB-D observations in the world frame. 
\end{itemize}

%% file: Review.tex
\section{Related Work}
\label{sec:review}


Several RGB-D SLAM approaches \cite{newcombe2011kinectfusion, endres2012evaluation, kerl2013dense, endres20133, whelan2016elasticfusion} are able to generate accurate trajectory and a dense 3D model of the environment. However, early RGB-D SLAM techniques focus on obtaining a geometric map and overlook the semantics. 
Later, object-level SLAM approaches \cite{nicholson2018quadricslam, yang2019cubeslam} are proposed to model both geometry and semantics. Those works focus on estimating the object pose accurately, but have limited capabilities to model object shape details due to the very simple geometric shape models used, such as cuboids and quadrics.  


Compared with other similar works~\cite{Mescheder_2019_CVPR, Chen_2019_CVPR} on learning implicit function for surface, DeepSDF \cite{park2019deepsdf} learns a continuous metric function of distance instead of binary classification function dividing inside or outside, which makes it suitable for gradient-based optimization in SLAM. 
Subsequent works along the direction of DeepSDF include FroDO \cite{runz2020frodo}, MOLTR \cite{li2020mo}, and DualSDF \cite{hao2020dualsdf}. 
FroDO leverages both point cloud and SDF representations, which defines sparse and dense losses to optimize the object shape. 
An extension of FroDO is MOLTR, which reconstructs an object shape by fusing multiple single-view shape codes to handle both static and dynamic objects. 
Similar to the coarse-to-fine shape estimation in FroDO and MOLTR, DualSDF uses two levels of granularity to represent 3D shapes. A shared latent space is employed to tightly couple the two levels, and a Gaussian prior is imposed on the latent space to enable sampling, interpolation, and optimization-based manipulation.  
DeepSDF and the derivatives offer models for accurate shape modeling but few of them consider object pose estimation. 

Object pose estimation is a critical step in the construction of an object level map. 
To estimate the transformation between world frame and the object frame, Scan2CAD \cite{avetisyan2019scan2cad} estimates the object pose and scale by establishing keypoint correspondences between the objects in the scene and their 3D CAD models. The keypoints are annotated for the CAD models and predicted by CNNs during inference. The Harris keypoints are detected from the 3D scan to be matched with those keypoints on the CAD models. However, both keypoint annotation and model retrieval take a long time for objects with complicated shapes, such as sofa. Later on Avetisyan \etal\cite{avetisyan2019end} dramatically increased the efficiency of the alignment process by utilizing a novel differentiable Procrustes alignment loss. Firstly, a proposed 3D CNN is used to identify objects in the 3D scan. Secondly, object bounding boxes are used to establish correspondence between scan objects and the CAD models. Lastly, alignment-informed correspondences are learnt via the differentiable Procrustes alignment loss.
Furthermore, multi-view constraints are introduced in Vid2CAD \cite{maninis2020vid2cad}. 

In the proposed ELLIPSDF, a learnt continuous SDF is used to reconstruct the object at arbitrary resolutions, and thus our approach has a more expressive object model in comparison to \cite{hu2019deep, sucar2020nodeslam}. 
Furthermore, our model has two levels of granularity that provide a coarse object prior to optimize the object scale, which is different from FroDO or \cite{afolabi2020extending}. 
Our system is online and more efficient, and unlike prior works that focus on single object estimation, we also present a large-scale, quantitative evaluation using a public benchmark that has multiple objects. 

%% file: Background.tex
\section{Background}
\label{sec:background}


Rigid body orientation, pose, and similarity are represented using the $\text{SO}(3)$, $\text{SE}(3)$, and $\text{SIM}(3)$ Lie groups, respectively, defined as:
\begin{gather}
\label{eq:LieGroups}
\scaleMathLine{\begin{aligned}
\text{SO}(3) &\triangleq \crl{ \bfR \in \bbR^{3 \times 3} \mid \bfR^\top\bfR = \bfI, \det(\bfR) = 1},\\
\text{SE}(3) &\triangleq \crl{ \begin{bmatrix} \bfR & \bft\\\mathbf{0}^\top & 1 \end{bmatrix} \in \bbR^{4 \times 4} \,\bigg\vert\, \bfR \in SO(3), \bft \in \bbR^3}, \\
\text{SIM}(3) &\triangleq \crl{ \begin{bmatrix} s\bfR & \bft\\\mathbf{0}^\top & 1 \end{bmatrix} \in \bbR^{4 \times 4} \,\bigg\vert\, \bfR \in SO(3), \bft \in \bbR^3, s \in \bbR}.
\end{aligned}}
\raisetag{10ex}
\end{gather}
We overload $\bfxi_\times$ to denote a mapping from a vector in $\bbR^3$ or $\bbR^6$ or $\bbR^7$ to the Lie algebra $\mathfrak{so}(3)$, $\mathfrak{se}(3)$, or $\mathfrak{sim}(3)$, associated with the Lie groups in \eqref{eq:LieGroups}, defined as:
\begin{gather}
\label{eq:LieAlgebras}
\begin{aligned}
\mathfrak{so}(3) &\triangleq \crl{\bfxi_\times = \begin{bmatrix}0 & -\xi_3 & \xi_2\\\xi_3 & 0 & -\xi_1\\-\xi_2 & \xi_1 & 0 \end{bmatrix} \,\bigg\vert\, \bfxi \in \bbR^3},\\
\mathfrak{se}(3) &\triangleq \crl{\bfxi_\times = \begin{bmatrix} \bftheta_\times & \bfrho \\ \mathbf{0}^\top & 0 \end{bmatrix}\,\bigg\vert\, \bfxi = \begin{bmatrix} \bfrho \\ \bftheta\end{bmatrix} \in \bbR^6},\\
\mathfrak{sim}(3) &\triangleq \crl{\bfxi_\times = \begin{bmatrix} \sigma \bfI + \bftheta_\times & \bfrho \\ \mathbf{0}^\top & 0 \end{bmatrix} \,\bigg\vert\, \bfxi = \begin{bmatrix} \bfrho \\ \bftheta \\ \sigma \end{bmatrix} \in \bbR^7}.
\end{aligned}
\raisetag{12ex}
\end{gather}
We define an infinitesimal change of a Lie group element $\bfT$ via a left perturbation $\exp\prl{\bfxi_\times}\bfT$, using the exponential map $\exp\prl{\bfxi_\times}$ to retract a Lie algebra element $\bfxi_\times$ to the Lie group. Please refer to \cite[Ch.7]{BarfootBook} or \cite{Gao2017SLAM} for details. 

The coarse shape of a rigid body is represented using a \emph{quadric shape} \cite[Ch.3]{MVGBook}, $\crl{ \bfx \in \bbR^3 \mid \underline{\bfx}^\top \bfQ \underline{\bfx} \leq 0}$, where $\underline{\bfx} \triangleq [\bfx^\top, 1]^\top$ denotes the homogeneous coordinates of $\bfx$ and $\bfQ \in \bbR^{4 \times 4}$ is a symmetric matrix. An axis-aligned ellipsoid centered at the origin:
\begin{equation}
\label{eq:ellipsoid}
\mathcal{E}_{\bfu} \triangleq \crl{\bfx \in \bbR^3 \mid \bfx^\top \bfU^{-\top}\bfU^{-1}\bfx \leq 1},
\end{equation}
where $\bfU \triangleq \diag(\bfu)$ and the elements of the vector $\bfu \in \bbR^3$ specify the lengths of the semi-axes of $\mathcal{E}_{\bfu}$. An ellipsoid $\mathcal{E}_{\bfu}$ is a special case of a quadric shape with $\bfQ = \diag(\bfU^{-2},-1)$. 
A quadric shape can also be defined in dual form, as the set of planes $\underline{\boldsymbol{\pi}} = \bfQ\underline{\bfx}$ that are tangent to the shape surface at each $\bfx$. This dual quadric surface definition is $\crl{ \bfpi \in \bbR^3 \mid \underline{\bfpi}^\top \bfQ^* \underline{\bfpi} = 0}$, where $\bfQ^* = \adj(\bfQ)$ is the adjugate of $\bfQ$.
A dual quadric defined by $\bfQ^*$ can be scaled, rotated, or translated by a similarity transform $\bfT \in \text{SIM}(3)$ as $\bfT \bfQ^* \bfT^\top$. Similarity, a dual quadric can be projected to a lower-dimensional space by a projection matrix $\bfP = \begin{bmatrix} \bfI & \mathbf{0} \end{bmatrix}$ as $\bfP \bfQ^* \bfP^\top$.

The fine shape of a rigid body is represented as $\crl{\bfx \in \bbR^3 \mid f(\bfx) \leq 0}$ using the \emph{signed distance field} of a set $\calS \subset \bbR^3$:
\begin{equation}
f(\bfx) = \begin{cases}
  -d(\bfx,\partial\calS), & \bfx \in \calS,\\
  \phantom{-}d(\bfx,\partial\calS), & \bfx \notin \calS,
\end{cases}
\end{equation}
where $d(\bfx,\partial\calS)$ denotes the Euclidean distance from a point $\bfx \in \bbR^3$ to the boundary $\partial\calS$ of $\calS$.

%% file: Problem.tex
\section{Problem Formulation}
\label{sec:problem}

Consider an RGB-D camera whose optical frame has pose $\bfC_k \in \text{SE}(3)$ with respect to the global frame at discrete time steps $k = 1,\ldots,K$. Assume that the camera is calibrated and its pose trajectory $\crl{\bfC_k}_k$ is known, e.g., from a SLAM or SfM algorithm. At time $k$, the camera provides an RGB image $I_k : \Omega^2 \mapsto \mathbb{R}_{\geq 0}^3$ and a depth image $D_k : \Omega^2 \mapsto \mathbb{R}_{\geq 0}$ such that $I_k(\bfp)$ and $D_k(\bfp)$ are the color and depth of a pixel $\bfp \in \Omega^2$ in normalized pixel coordinates. The camera moves in an unknown environment that contains $N$ objects $\calO \triangleq \crl{\bfo_n}_{n=1}^N$. Each object $\bfo_n = (\bfc_n,\bfi_n)$ is an instance $\bfi_n$ of class $\bfc_n$, defined below.

\begin{definition*}
An \emph{object class} is a tuple $\bfc \triangleq \prl{\nu, \bfz, f_{\bftheta}, g_{\bfphi}}$, where $\nu \in \mathbb{N}$ is the class identity, e.g., chair, table, sofa, and $\bfz \in \mathbb{R}^d$ is a latent code vector, encoding the average class shape. The class shape is represented in a canonical coordinate frame at two levels of granularity: coarse and fine. The coarse shape is specified by an ellipsoid $\calE_\bfu$ in \eqref{eq:ellipsoid} with semi-axis lengths $\bfu = g_{\bfphi}(\bfz)$ decoded from the latent code $\bfz$ via a function $g_{\bfphi} : \bbR^d \mapsto \bbR^3$ with parameters $\bfphi$. The fine shape is specified by the signed distance $f_{\bftheta}(\bfx,\bfz)$ from any $\bfx \in \bbR^3$ to the average shape surface, decoded from the latent code $\bfz$ via a function $f_{\bftheta} : \bbR^3 \times \bbR^d \mapsto \bbR$ with parameters $\bftheta$.
\end{definition*}

\begin{definition*}
An \emph{object instance} of class $\bfc$ is a tuple $\bfi \triangleq \prl{\bfT, \delta\bfz}$, where $\bfT \in \text{SIM}(3)$ specifies the transformation from the global frame to the object instance frame, and $\delta\bfz \in \bbR^d$ is a deformation of the latent code $\bfz$, specifying the average shape of class $\bfc$.
\end{definition*}

\begin{figure*}[t] 
  \centering
  \includegraphics[width=\linewidth]{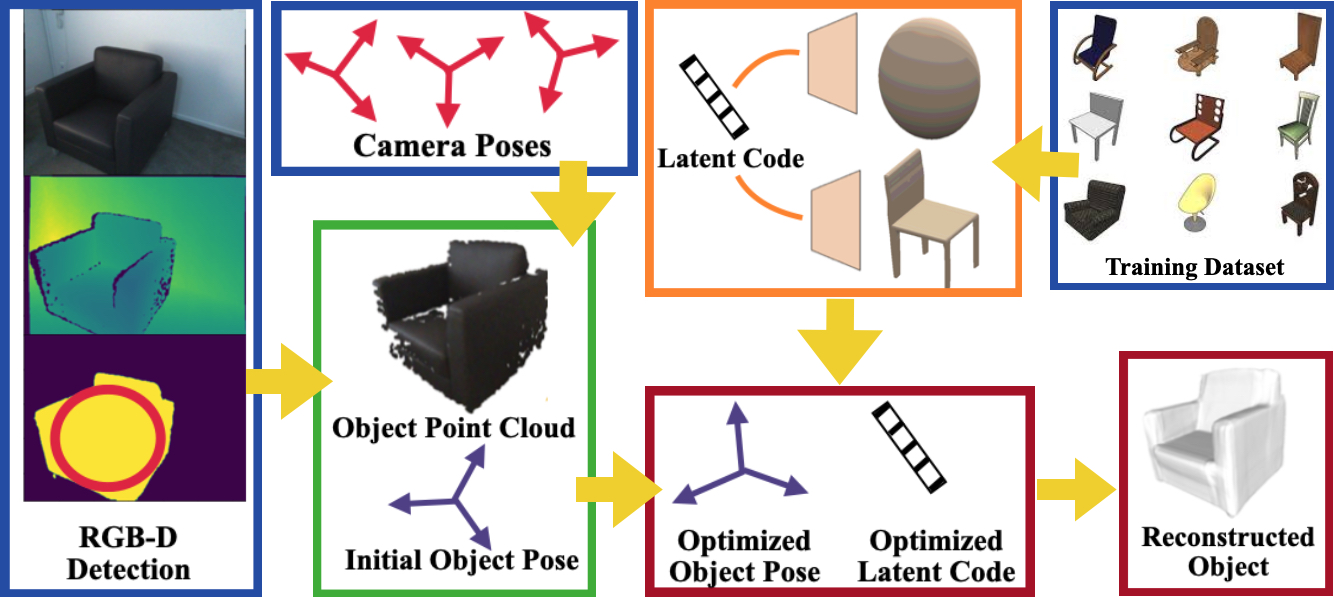}
  \caption{ELLIPSDF Overview: A point cloud and initial pose (\textit{green}) are obtained from RGB-D detections of a chair instance from known camera poses (\textit{blue}). A bi-level category shape description, consisting of a latent shape code, a coarse shape decoder, and a fine shape decoder (\textit{orange}), is trained offline using a dataset of mesh models. Given the observed point cloud, the pose and shape deformation of the newly seen instance are optimized jointly online, achieving shape reconstruction in the global frame (\textit{red}).}
  \label{fig:framework}
\end{figure*}

We assume that object detection (e.g., \cite{Cai2019Cascade}) and tracking (e.g., \cite{bewley2016simple}) algorithms are available to provide the class $\bfc_n$ and pixel-wise segmentation $\Omega_{n,k}^2 \subseteq \Omega^2$ of any object $n$ observed by the camera at time $k$. 
Our goal is to estimate the transformation and shape $\bfi_n := (\bfT_n, \delta \bfz_n)$ of each observed object $n$. We consider object instances independently and drop the subscript $n$ when it is clear from the context. 


Given an object with multi-view segmentation $\Omega_{k}^2$, we use the depth $D_k(\bfp)$ of each pixel $\bfp \in \Omega_{k}^2$ to obtain a set of points $\calX_k(\bfp)$ along the ray starting from the camera optical center and passing through $\bfp$. The sets $\calX_k(\bfp)$ is used to optimize the pose and shape of the object instance. For each ray, we choose three points, one lying on the observed surface, one a small distance $\epsilon>0$ in front of the surface, and one a small distance $\epsilon$ behind. Given $d \in \{0,\pm \epsilon\}$, we obtain points $\bfy \in \bbR^3$ in the optical frame corresponding to the pixels $\bfp \in \Omega_{k}^2$:
\begin{equation*}
\scaleMathLine{\calY_k(\bfp) \triangleq \crl{(\bfy, d) \,\bigg\vert\, \bfy = \prl{D_k(\bfp) + \frac{d}{\|\underline{\bfp}\|}}\underline{\bfp}, \; d \in \{0,\pm \epsilon\}},}
\end{equation*}
and project them to the global frame using the known camera pose $\bfC_k$:
\begin{equation}\label{eq:distance_measurements}
\calX_k(\bfp) \triangleq \crl{(\bfx, d) \,\bigg\vert\, \bfx = \bfP \bfC_k \underline{\bfy}, \; (\bfy,d) \in \calY_k(\bfp)}.
\end{equation}
%

We define an error function $e_{\bfphi}$ to measure the discrepancy between a distance-labelled point $(\bfx,d) \in \calX_{k}(\bfp)$ observed close to the instance surface and the coarse shape $\calE_{\bfu}$ provided by $\bfu = g_{\bfphi}(\bfz)$. Another error function $e_{\bftheta}$ is used for the difference between $(\bfx,d)$ and the SDF value $f_{\bftheta}(\bfx, \bfz)$ predicted by the fine shape model. The overall error function is defined as: 
\begin{align}
\label{eq:cost_function}
&e(\bfT,\delta \bfz, \bftheta, \bfphi ; \crl{\calX_k(\bfp)}) \triangleq \alpha e_r(\delta \bfz) \\
&+ \sum_{k=1}^K 
      \sum_{\bfp \in \Omega_{k}^2}
      \sum_{(\bfx,d) \in \calX_k(\bfp)} \!\!\!\beta e_{\bftheta}(\bfx,d,\bfT,\delta \bfz) + \gamma  e_{\bfphi}(\bfx,d,\bfT,\delta \bfz)\notag,
\end{align}
where $e_r(\delta \bfz)$ is a shape deformation regularization term. The coarse-shape error, $e_{\bftheta}$, fine-shape error, $e_{\bfphi}$, and the regularization, $e_r$ are defined precisely in Sec. \ref{sec:train_code}.

We distinguish between a training phase, where we optimize the parameters $\bfz$, $\bftheta$, $\bfphi$ of an object class using offline data from instances with known mesh shapes, and a testing phase, where we optimize the pose $\bfT$ and shape deformation $\delta \bfz$ of a previously unseen instance from the same category using online distance data from an RGB-D camera.

In training, we generate points $\crl{\calX_{n,k}(\bfp)}$ close to the surface of each available mesh model $n$ in a canonical coordinate frame (with identity pose $\bfI_4$) and optimize the class shape parameters via:
\begin{equation}
\min_{\crl{\delta \bfz_n}, \bftheta, \bfphi} \sum_n e(\bfI_4,\delta \bfz_n, \bftheta, \bfphi ; \crl{\calX_{n,k}(\bfp)}).
\end{equation}


In testing, we receive points $\crl{\calX_{k}(\bfp)}$ in the global frame, generated by the RGB-D camera from the surface of a previously unseen instance. Assuming known object class, we fix the trained shape parameters $\bfz^*$, $\bftheta^*$, $\bfphi^*$ and optimize the unknown instance transform $\bfT \in \text{SIM}(3)$ and shape deformation $\delta \bfz \in \bbR^d$:
\begin{equation} \label{eq:test_optimization}
\min_{\bfT, \delta \bfz} \; e(\bfT,\delta \bfz, \bftheta^*, \bfphi^* ; \crl{\calX_k(\bfp)}).
\end{equation}

%% file: Reconstruction.tex
\section{Object Pose and Shape Optimization}
\label{sec:reconstruction}

This section develops ELLIPSDF, an autodecoder model for bi-level object shape representation. Sec.~\ref{sec:train_code} presents the model and defines the error functions for its parameter optimization. Sec.~\ref{sec:shape_pose_inference} describes how a trained ELLIPSDF model is used at test time for multi-view joint optimization of object pose and shape. An overview is shown in Fig.~\ref{fig:framework}.




\subsection{Training an ELLIPSDF Model}
\label{sec:train_code}


{\vspace{1ex}\bf \noindent Bi-level Shape Representation: }%
The ELLIPSDF shape model consists of two autodecoders $g_{\bfphi}(\bfz)$ and $f_{\bftheta}(\bfx,\bfz)$, using a shared latent code $\bfz \in \bbR^d$. The first autodecoder provides a \emph{coarse} shape representation with parameters $\bfphi$, as an axis-aligned ellipsoid $\calE_{\bfu}$ in a canonical coordinate frame with semi-axis lengths $\bfu = g_{\bfphi}(\bfz)$. The second autoencoder provides a \emph{fine} shape representation with parameters $\bftheta$, as an implicit SDF surface $\crl{\bfx \in \bbR^3 \mid f_{\bftheta}(\bfx,\bfz) \leq 0}$ in the same canonical coordinate frame. We implement $g_{\bfphi}(\bfz)$ and $f_{\bftheta}(\bfx,\bfz)$ as $8$-layer perceptrons with one cross-connection, as described in Sec.~D in the supplementary material of DualSDF \cite{hao2020dualsdf}. The reparametrization trick \cite{kingma2013auto} is used to maintain a Gaussian distribution $\bfz = \bfmu + \diag(\bfsigma) \bfepsilon$ over the latent code with $\bfepsilon \sim \mathcal{N}(\bf{0},\bf{I})$. Thus, at training time, the ELLIPSDF model parameters are the mean $\bfmu \in \bbR^d$ and standard deviation $\bfsigma \in \bbR^d$ of the latent shape code and the coarse and fine shape autodecoder parameters $\bfphi$ and $\bftheta$. The model is visualized in Fig.~\ref{fig:two_level_model}. 





{\vspace{1ex}\bf \noindent Error Functions: }%
We introduce error terms that play a key role for optimizing the category-level latent code $\bfz$ and decoder parameters $\bftheta$, $\bfphi$, during training time, as well as the transformation $\bfT$ from the global frame to the canonical object frame and the latent code deformation $\delta\bfz$ of a particular instance during test time. The training data for an ELLIPSDF model consists of distance-labeled point clouds $\calX_{n,k}(\bfp)$ associated with instances $n$ from the same class, as introduced in Sec.~\ref{sec:problem}. A different latent code $\bfz_n$ is optimized for each instance $n$, while the decoder parameters $\bftheta$ and $\bfphi$ are common for all instances of the same class. 





The fine-level shape error function $e_{\bftheta}(\bfx,d,\bfT,\delta\bfz)$ of a point $\bfx$ in global coordinates  with signed distance label $d$ is defined as:
\begin{equation}
  \label{eq:e_k}
  e_{\bftheta}(\bfx,d,\bfT,\delta\bfz) \triangleq \rho(s f_{\bftheta}(\bfP\bfT \underline{\bfx}; \bfz+\delta\bfz) - d).
\end{equation}
In the definition above, the point $\bfx$ is first transformed to the object coordinate frame via $\bfP\bfT \underline{\bfx}$ and the fine-shape model $f_{\bftheta}$ is queried with the instance shape code $\bfz + \delta\bfz$ to predict the SDF to the object surface. Since SDF values vary proportionally with scaling \cite{afolabi2020extending}, the returned value is scaled back by $s$ before measuring its discrepancy with the label $d$. Instead of measuring the difference between $s f_{\bftheta}$ and $d$ in absolute value, we employ a Huber term \cite{Huber1964Robust} to make the error function robust against outliers:
\begin{equation}
\label{eq:huber_loss}
\rho(r) \triangleq 
\begin{cases}
\frac{1}{2}r^2 & |r|\leq \delta,\\
\delta(|r|-\frac{1}{2}\delta) & \text{else}.
\end{cases}
\end{equation}
Note that the error $e_{\bftheta}$ relates both the object pose and shape to the SDF residual, which is unique to our formulation and enables their joint optimization.

The coarse-level shape error function $e_{\bfphi}(\bfx,d,\bfT,\delta\bfz)$ is defined similarly, using a signed distance function for the coarse shape. Since the coarse shape decoder, $\bfu = g_{\bfphi}(\bfz)$, provides an explicit ellipsoid description, we first need a conversion to SDF before we can define the error term. An approximation of the SDF of an ellipsoid $\calE_{\bfu}$ with semi-axis lengths $\bfu$ can be obtained as:
\begin{equation}
  \label{eq:ellpsoid_sdf}
  h\left(\bfx, \bfu\right)
  =
  \frac{\left\|\bfU^{-1}\bfx\right\|_{2}\left(\left\|\bfU^{-1}\bfx\right\|_{2}-1\right)}
  {\left\|\bf{U}^{-2}\bfx\right\|_{2}}.
\end{equation}
Then, the coarse-level shape error of a point $\bfx$ in global coordinates  with signed distance label $d$ is defined as:
\begin{equation}
  \label{eq:e_g}
  e_{\bfphi}(\bfx,d,\bfT,\delta\bfz) \triangleq \rho(s h(\bfP\bfT \underline{\bfx}, g_{\bfphi}(\bfz+\delta\bfz)) - d).
\end{equation} 
%

During training, the object transformation is fixed to be the canonical coordinate frame $\bfT = \bfI_4$ because the training point-cloud data is collected directly in the object frame. The regularization term $e_r(\delta\bfz)$ in \eqref{eq:cost_function} is defined as the KL divergence between the distribution of $\delta\bfz$ and a standard normal distribution \cite{hao2020dualsdf}.

\begin{figure}[t]
    \centering
    \includegraphics[width=\linewidth]{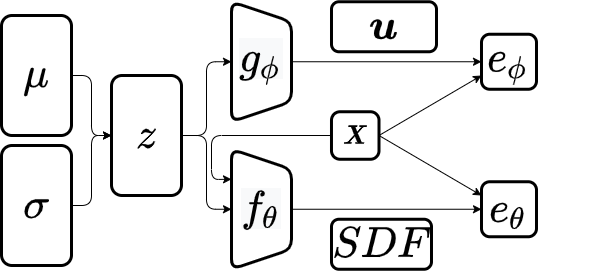}
    \caption{Overview of our ELLIPSDF bi-level object shape model. A latent shape code, $\bfz$, with distribution $\calN(\bfmu,\diag(\bfsigma)^2)$ is shared by a coarse shape decoder $g_\phi$, providing an ellipsoid shape description, and a fine shape decoder $f_\theta$, providing an SDF shape description. During training, the decoder parameters $\phi$ and $\theta$ are optimized by minimizing the errors between the SDF values of the training points $\bfx$, obtained close to the object surface, and the coarse and fine shape models.}
    \label{fig:two_level_model}
\end{figure}

\subsection{Joint Pose and Shape Optimization with an ELLIPSDF Model}
\label{sec:shape_pose_inference}

This section describes how a trained ELLIPSDF model is used to initialize and optimize the pose and shape of a new object instance at test time.


{\vspace{1ex}\bf \noindent Initialization: }%
We follow \cite{crocco2016structure, rubino20173d, gay2018visual} to initialize the $\text{SIM}(3)$ scale and pose of an observed object, relying on its coarse ellipsoid shape representation. We fit ellipses to the pixel-wise segmentation $\Omega_k^2$ of an object at each time $k$:
\begin{equation} \label{eq:ellipse_fit}
    \crl{ \bfq \in \Omega^2 \mid (\bfq - \bfc_k)^\top \bfE_k^{-1} (\bfq - \bfc_k) \leq 1},
\end{equation}
where the center and symmetric matrix are obtained as $\bfc_k = \frac{1}{|\Omega_k^2|}\sum_{\bfp \in \Omega_k^2} \bfp$ and $\bfE_k = \frac{2}{|\Omega_k^2|}\sum_{\bfp \in \Omega_k^2} (\bfp-\bfc_k)(\bfp-\bfc_k)^\top$. The axes lengths are the eigenvalues $\lambda_{0}$, $\lambda_{1}$ of $\bfE_k$. The 2D quadric surface corresponding to the ellipse in \eqref{eq:ellipse_fit} and its dual are defined by the matrix $\bfH_k$ and its inverse $\bfH_k^*$:
\begin{equation*}
    \scaleMathLine{\bfH_k = \begin{bmatrix} \bfE_k^{-1} & -\bfE_k^{-1}\bfc_k \\ -\bfc_k^\top \bfE_k^{-1} & \bfc_k^\top \bfE_k^{-1} \bfc_k - 1\end{bmatrix}, \;\; \bfH_k^* = \begin{bmatrix}\bfE_k - \bfc_k\bfc_k^\top & -\bfc_k \\ -\bfc_k^\top & -1\end{bmatrix}.}
\end{equation*}

An ellipsoid in dual quadric form $\bfQ^*$ in global coordinates and its conic projection $\bfH_k^*$ in image $k$ are related by $\beta_{k} \mathbf{H}_{k}^*=\mathbf{P} \bfC_k^{-1} \mathbf{Q}^* \bfC_k^{-\top} \mathbf{P}^{\top}$ defined up to a scale factor $\beta_{k}$. This equation can be rearranged to $\beta_{k} \mathbf{h}_{k} = \mathbf{G}_k\mathbf{v}$, where $\mathbf{h}_{k} = \operatorname{vech}(\mathbf{H}_{k}^*)$, $\mathbf{h}_{k} \in \mathbb{R}^6$, $\mathbf{v} = \operatorname{vech}(\mathbf{Q}^*)$ and $\mathbf{v} \in \mathbb{R}^{10}$. The operator $\operatorname{vech}$ serializes the lower triangular part of a symmetric matrix, and $\mathbf{G}_k$ is a matrix that depends on $\mathbf{P} \bfC_k^{-1}$. The explicit form of $\mathbf{G}_k$ is derived in (5) in \cite{rubino20173d}. 
Next, a least squares system is constructed from the multi-view observations. By stacking all observations, we obtain $\mathbf{M} \mathbf{w} = \bf0$, where $\mathbf{w} = (\mathbf{v}, \beta_1,\ldots,\beta_k)^\top$, and $\bfM$ is defined in (8) in  \cite{rubino20173d}.  
This leads to a least squares system:
\begin{equation}
\label{eq:ellipsoid_lsq}
\hat{\mathbf{w}}=\arg \min _{\bfw}\left\|\mathbf{M} \mathbf{w}\right\|_{2}^{2} \quad \text { s.t. } \quad\|\mathbf{w}\|_{2}^{2}=1,
\end{equation}
which can be solved by applying SVD to $\mathbf{M}$, and taking the right singular vector associated to the minimum singular value. The constraint $\|\mathbf{w}\|_{2}^{2}=1$ avoids a trivial solution. The first $10$ entries of $\hat{\mathbf{w}}$ are $\hat{\mathbf{v}}$, which is a vectorized version of the dual ellipsoid $\hat{\mathbf{Q}}^*$ in the global frame. To avoid degenerate quadrics, a variant of the least squares system in \eqref{eq:ellipsoid_lsq} is proposed in \cite{gay2018visual}, which constrains the center of the ellipse and the reprojection of the center of the 3D ellipsoid to be close. Thus, we modify $\bfM$ using the version in (9) in \cite{gay2018visual} to improve the estimation.

The object pose $\hat{\bfT}^{-1}$ can be recovered by relating the estimated ellipsoid $\hat{\mathbf{Q}}^*$ in global coordinates to the ellipsoid $\bfQ^*_{\bfu}$ in the canonical coordinate frame predicted by the coarse shape decoder $\bfu = g_{\bfphi}(\bfz)$ using the average class shape $\bfz$: 
\begin{equation*}
\begin{aligned}
\hat{\bfQ}^*\! =
\hat{\mathbf{T}}^{-1} \bfQ_{\bfu}^* \hat{\mathbf{T}}^{-\top}\!\!=
\begin{bmatrix} 
\hat{s}^2 \hat{\mathbf{R}} \bfU\bfU^\top \hat{\mathbf{R}}^\top -  \hat{\bft} \hat{\bft}^\top & - \hat{\bft} \\ -\hat{\bft}^\top & -1
\end{bmatrix}.
\end{aligned}
\end{equation*}
The translation $\hat{\bft}$ can be recovered from the last column of $\hat{\bfQ}^*$.
To recover the rotation, note that $\bfA \triangleq \bfP\hat{\mathbf{Q}}^*\bfP^\top  + \hat{\bft} \hat{\bft}^\top = \hat{s}^2\hat{\bfR}\bfU\bfU^\top\hat{\bfR}^\top$ is a positive semidefinite matrix. Let its eigen-decomposition be $\bfA = \bfV\bfY\bfV^\top$, where $\bfY$ is a diagonal matrix containing the eigenvalues of $\bfA$. Since $\bfU\bfU^\top$ is diagonal, it follows that $\hat{\bfR} = \bfV$, while the scale $\hat{s}$ is obtained as $\hat{s} = \frac{1}{3} \sqrt{\tr(\bfU^{-1}\bfY \bfU^{-\top})}$.
%
%
Note that although the $\text{SIM}(3)$ pose could also be recovered from the object point cloud, other outlier rejection methods are required \cite{wu2020eao} when the point cloud is noisy.


{\vspace{1ex}\bf \noindent Optimization: }%
Given the initialized instance transformation $\hat{\bfT}$ and initial shape deformation $\delta\hat{\bfz} = \mathbf{0}$, we solve the joint pose and shape optimization in \eqref{eq:test_optimization} via gradient descent. Note that the decoder parameters $\bftheta$, $\bfphi$ and the mean category shape code $\bfz$ are fixed during online inference. Since $\bfT$ is an element of the $\text{SIM}(3)$ manifold, the gradients and gradient steps need to be computed by projecting to the tangent $\text{sim}(3)$ vector space and retracting back to $\text{SIM}(3)$. We introduce local perturbations $\bfT = \exp\prl{\bfxi_\times} \hat{\bfT}$, $\delta\bfz = \delta\tilde{\bfz} + \delta\hat{\bfz}$ and derive the Jacobians of the error function in \eqref{eq:cost_function} with respect to $\bfxi$ and $\delta\tilde{\bfz}$.


\begin{proposition}
\label{prop:sdf-sim3-jacobians}
The Jacobian of $e_{\bftheta}$ in \eqref{eq:e_k} with respect to the transformation perturbation $\bfxi \in \mathfrak{sim}(3)$ is:
\begin{equation*}
\scaleMathLine[1]
{
  \begin{aligned}
    \frac{\partial e_{\bftheta}}{\partial \bfxi} 
    &= 
    \frac{\partial \rho(r)}{\partial r}
    \prl{
      \hat{s} [{\bf0}_6,1] f_{\bftheta} (\bfx,\delta\hat{\bfz})
      + 
      \hat{s} \nabla_{\bfx} f_{\bftheta}(\bfx,\delta\hat{\bfz})^\top
      \bfP \brl{\hat{\bfT} \underline{\bfx}}^\odot 
    }
    \\
    \frac{\partial e_{\bftheta}}{\partial \delta\tilde{\bfz}} 
    &= 
    \frac{\partial \rho(r)}{\partial r}
    \hat{s} \nabla_{\bfz} f_{\bftheta}(\bfx,\delta\hat{\bfz}), 
    \end{aligned}
}
\end{equation*}
where $f_{\bftheta}(\bfx,\delta\hat{\bfz}) = f_{\bftheta}(\bfP\hat{\bfT} \underline{\bfx}; \bfz+\delta\hat{\bfz})$ is defined in \eqref{eq:e_k} and $\frac{\partial \rho(r)}{\partial r}$ is the derivative of the Huber term in \eqref{eq:huber_loss} evaluated at $r = \hat{s} f_{\bftheta}(\bfx,\delta\hat{\bfz}) - d$:
\[
  \frac{\partial \rho(r)}{\partial r}
  =\left\{\begin{array}{ll}
    r & |r| \leq \delta \\
    \text{sign}(r)\delta  & \text{ else}. 
    \end{array}\right.
\] 
The operator $\underline{\bfx}^\odot$ is defined as:
\begin{equation*}
\underline{\bfx}^\odot \triangleq \begin{bmatrix} \bfI_3 & -\bfx_\times & \bfx\\ \mathbf{0}^\top & \mathbf{0}^\top & 0 \end{bmatrix} \in \mathbb{R}^{4 \times 7}. 
\end{equation*}
\end{proposition}

\begin{proof}
Using the chain rule and the product rule:
\begin{equation*}
\begin{aligned}
  \frac{\partial e_{\bftheta}}{\partial\bfxi} 
  = 
  \frac{\partial e_{\bftheta}}{\partial r}
  \frac{\partial r}{\partial\bfxi}
  = 
    \frac{\partial e_{\bftheta}}{\partial r}
    \prl{
      \frac{\partial s}{\partial\bfxi}
      f_{\bftheta} (\bfx,\delta\bfz)
      +
      s
      \frac{\partial f_{\bftheta}}{\partial{}^{}_{O}\bfx}
      \frac{\partial {}^{}_{O}\bfx}{\partial\bfxi}
    }, 
\end{aligned}
\end{equation*}
where ${}^{}_{O}\bfx = \bfP\bfT \underline{\bfx}$ is a point in the object frame. We have $
\frac{\partial s}{\partial\bfxi}
  = 
  e^\sigma [{\bf0}_6,1]
  = s [{\bf0}_6,1]
$. 
The term $s\frac{\partial f_{\bftheta}}{\partial{}^{}_{O}\bfx}$ is the gradient of the fine-level SDF decoder with respect to the input $s \nabla_{\bfx} f_{\bftheta}(\bfx, \delta\bfz)$, which could be obtained from auto-differentiation. Finally, we have:
\begin{equation*}
\begin{aligned}
{}^{}_{O} \bfx &= \bfP \bfT \underline{\bfx} \approx 
\bfP (\bfI + \bfxi_\times) \hat\bfT \underline{\bfx} \\
&=
\bfP \hat\bfT \underline{\bfx}
+
\bfP \bfxi_\times \hat\bfT \underline{\bfx} \\ 
&= 
\bfP \hat\bfT \underline{\bfx}
+
\underbrace{\bfP [\hat\bfT \underline{\bfx}]^\odot}_{\frac{\partial {}^{}_{O} \bfx}{\partial \bfxi}} \bfxi. 
\end{aligned}\qedhere
\end{equation*}
\end{proof}

In the second equality in Prop.~\ref{prop:sdf-sim3-jacobians}, the term $\frac{\partial \rho(r)}{\partial r} \hat{s} \nabla_{\bfz} f_{\bftheta}(\bfx, \delta\hat{\bfz})$ is the gradient of the fine-level SDF loss with respect to the input $\bfz$ and can be obtained via auto-differentiation. The Jacobians of the coarse-level SDF error $\frac{\partial e_{\bfphi}}{\partial \bfxi}$, $\frac{\partial e_{\bfphi}}{\partial \delta \tilde{\bfz}}$ can be obtained in a similar way. 


After obtaining the Jacobians, the pose and latent shape code can be optimized via:
\begin{equation*}
\begin{aligned}
\bfT^{i+1} &\triangleq 
\exp \prl{- \eta_1 
\frac{\partial e(\bfT,\delta \bfz, \bftheta^*, \bfphi^* ; \crl{\calX_k(\bfp)})}{\partial \bfxi}}
\bfT^{i}
\\
\delta \bfz^{i+1} &\triangleq 
\delta \bfz^{i} - \eta_2 
\prl{\frac{\partial e(\bfT,\delta \bfz, \bftheta^*, \bfphi^* ; \crl{\calX_k(\bfp)})}{\partial \delta \bfz}}, 
\end{aligned}
\end{equation*}
where $\eta_1, \eta_2$ are step sizes, $\delta \bfz^0 = \mathbf{0}$, and $\bfT^0 = \hat{\bfT}$ is obtained from the initialization. During optimization, we add regularization $e_r(\delta\bfz) = \|\delta\bfz\|_2^2$ to restrict the amount of latent code deformation.

%% file: Evaluation.tex
\section{Evaluation}
\label{sec:evaluation}

\begin{figure*}[t]
    \centering
    \includegraphics[width=0.85\linewidth]{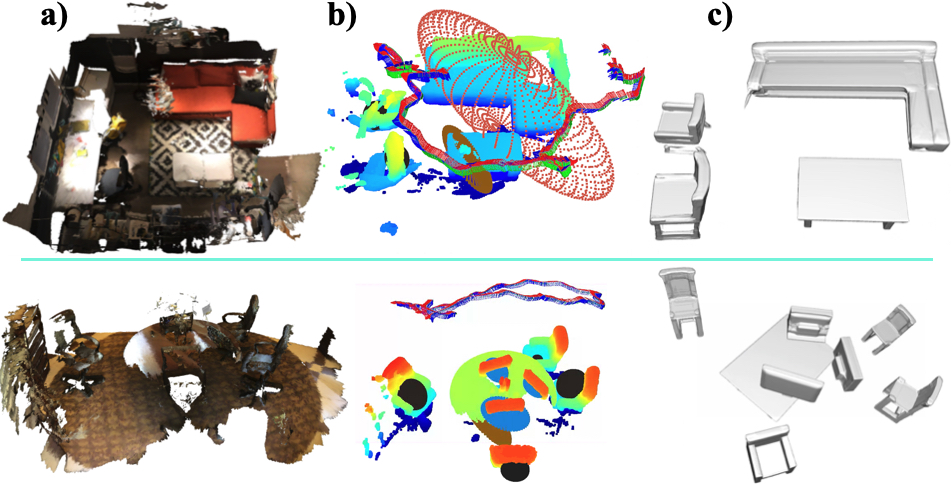}
    \caption{Qualitive results. Column a): Ground-truth scene in ScanNet Sequence $0518$ (upper row) and $0314$ (lower row). Column b): The RGB axes are the camera trajectory, point clouds are the ones obtained from RGB-D sensor with added pesudo points, and the ellipsoids (black for chair, red for sofa, blue for monitor, brown for table) are the initialized objects. Column c): Reconstructed meshes using ELLIPSDF, rendered from the optimized latent code and pose.}
    \label{fig:qual_results}
  \end{figure*}

\subsection{Training Details}
\label{sec:training_details}
The ELLIPSDF decoder model is trained on synthetic CAD models from ShapeNet \cite{chang2015shapenet}. Each model's scale is normalized to be inside a unit sphere. We sample points and calculate their SDF values using a uniform distribution in the unit sphere for training the coarse-level shape decoder $g_{\bfphi}$. Another set of points that are close to the model surface are sampled for training the fine-level shape decoder $f_{\bftheta}$.

The following setting were used to train the decoder networks and the latent shape code $\bfz$. We use the Adam optimizer with initial learning rate $5\times 10^{-4}$,  $0.5$ ratio decay every 300/700 epochs for the coarse and fine level networks separately. 
The total epoch number is $1500$. 
The latent code dimension is $64$, and the network structure follows the model in DualSDF \cite{hao2020dualsdf}. 

\subsection{Qualitative Results}

We evaluate ELLIPSDF on the ScanNet dataset \cite{dai2017scannet}, which provides 3D scans captured by a RGB-D sensor of indoor scenes with chairs, tables, displays, etc. We segment out objects from scene-level mesh using provided instance labels, and sample points from object meshs to generate point observations. 
Visualizations of shape optimization for a chair are shown in Fig.~\ref{fig:fig1_refined}. Optimization step improves the scale and shape estimates notably, e.g. by transforming the four-leg mean shape into an armchair. 
Larger scale qualitative results are shown in Fig.~\ref{fig:qual_results}, demonstrating the effectiveness of joint shape and pose optimization. Optimized poses are closer to the ground-truth, and optimized shapes resemble the objects better than simple primitive shapes such as cuboids or quadrics that lacks fine details. 
For example, the successful reconstruction of an angle sofa is illustrated in the upper row in Fig.~\ref{fig:qual_results}, which deforms from an initial mean sofa shape that does not have an angle. 
\begin{figure}[t]
    \centering
    \includegraphics[width=1\linewidth]{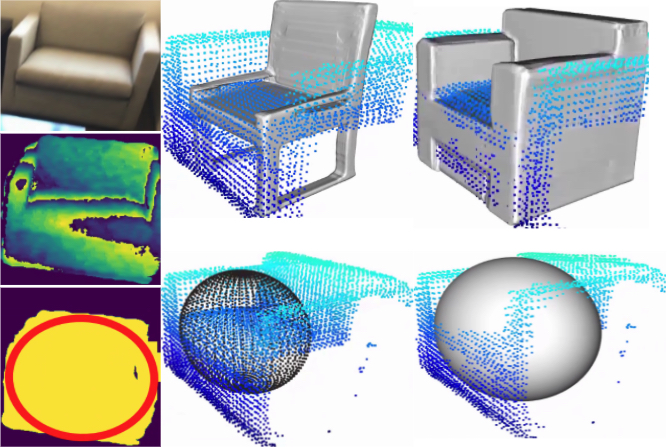}
    \caption{Intermediate ELLIPSDF stages. First column: RGB image, depth image, instance segmentation (yellow), fitted ellipse (red) for a chair in ScanNet scene $0461$. Second column: mean shape and ellipsoid with initialized pose. Third column: optimized fine-level and coarse-level shapes with optimized pose.}
    \label{fig:fig1_refined}
\end{figure}
ELLIPSDF is also able to deal with partial observations as seen in the lower row in Fig.~\ref{fig:qual_results}. Although the observed point clouds of the displays and the chairs are sparse, our approach still reconstructs those objects successfully. Nevertheless, the reconstruction is a square instead of rounded for the table due to a severe occlusion of the observation that only less than half of the table is observed. 

\subsection{Quantitative Results}

This section presents quantitative evaluation against other methods regarding both pose and shape estimation accuracy. We also present ablation studies to showcase the improvement of the optimization over initialization-only results, and the bi-level model over a one level model. 


\begin{table}[t]
\centering
\caption{Quantitative results for pose estimation on ScanNet~\cite{dai2017scannet}.}
\scalebox{0.78}{
    \begin{tabular}{c|c|c|c} 
    \hline
    Scan2CAD \cite{avetisyan2019scan2cad} & Vid2CAD \cite{maninis2020vid2cad} & ELLIPSDF (init) & ELLIPSDF (opt) \\
    \hline
    31.7 & 38.3 & 31.5 & \textbf{39.6} \\
    \hline
    \end{tabular}}
    \label{tab:pose}
\end{table}

{\vspace{1ex}\bf \noindent Evaluation on Object Pose: }%
We obtain the ground-truth object pose annotations from Scan2CAD \cite{avetisyan2019scan2cad} and follow the pose evaluation metrics it defines, which decomposes a pose $\bfT \in \text{SIM}(3)$ into rotation $\bfq$, translation $\bfp$ and scale $\bfs$. For an accurate pose estimation, the error thresholds for translation, rotation, and scales are set as $0.2$, $20^{\circ}$ and $20\%$ respectively with respect to the ground-truth pose.
The pose evaluation is presented in Tab.~\ref{tab:pose}, in which ELLIPSDF (init) refers to the initialization-only step in Sec.~\ref{sec:shape_pose_inference}, whereas ELLIPSDF (opt) refers using both the initialization and optimization steps in Sec.~\ref{sec:shape_pose_inference}.
The last two columns in Tab.~\ref{tab:pose} show that adding optimization step using SDF residuals improves the estimation by the initialization-only variant, due to the additional SDF residuals to help estimate pose. 
Moreover, ELLIPSDF (opt) outperforms both Scan2CAD and Vid2CAD, which demonstrates the superiority of ELLIPSDF that employs a primitive ellipsoid shape tailored for pose and scale estimation.


{\vspace{1ex}\bf \noindent Evaluation on Object Shape: }%
We evaluate ELLIPSDF for shape prediction on ScanNet \cite{dai2017scannet} dataset in Tab.~\ref{tab:fitting_performance}. Instead of single object evaluation in FroDO \cite{runz2020frodo}, we evaluate on multiple objects, which is harder than the single-object-scene due to clustering and partial observations. The large scale evaluation verifies that our method can generalize across different sequences and objects.
The object point cloud sampled from the object mesh from~\cite{avetisyan2019scan2cad} is used as the ground truth $\calS_{gt}$, and the estimated point cloud $\calS_{est}$ is generated from the optimized latent code $\bf{z} + \delta \bf{z}$. 
Given the ground-truth point cloud $\calS_{gt}$ and ELLIPSDF point cloud $\calS_{est}$ for an object, the fitting rate with inlier ratio is
\begin{equation} 
\begin{aligned} \label{eq:fitting}
fit(\calS_{est}, \calS_{gt}) &= \frac{\vert \calS_{close} \vert}{\vert \calS_{est} \vert}, \\
\calS_{close} &= \{\bfv \in \calS_{est}: d_f(\bfv, \calS_{gt}) < \lambda \}, 
\end{aligned}
\end{equation}
where $\lambda = 0.2 (m)$. A distance function $d_f(\cdot, \cdot)$ is utilized to measure the distance between a point $\bfv$ and a point cloud $\calS$, which is the distance from the closest point $\bfu \in \calS$ to the point $\bfv$. In CAD-Deform \cite{ishimtsev2020cad}, the distance function is set to be L1 distance, while we use L2 distance.

\begin{table}[t]
    \centering
    \caption{Quantitative results for shape evlaution on ScanNet\cite{dai2017scannet}.}\label{tab:fitting_performance}
    \scalebox{0.8}{
    \begin{tabular}{l|c|c|c|c|c}
    \hline
     \textbf{Method} & cabinet & chair & display & table   & avg. \\ 
    \# intances  & 132 & 820 & 209 & 146   & 327  \\ \hline\hline
    ELLIPSDF (fine)  & 88.4 & 88.3 & 90.6 & 76.2       & 85.9 \\ 
    ELLIPSDF (coarse+fine)  & \textbf{91.0} & \textbf{90.6} & \textbf{96.9} & \textbf{77.3}   & \textbf{89.0} \\ \hline
    \end{tabular}
    }
\end{table}

We run ELLIPSDF (fine) and ELLIPSDF (coarse+fine) on 150 validation sequences on ScanNet \cite{dai2017scannet}, where ELLIPSDF (fine) means only the fine level SDF residual is used by setting $\gamma = 0$ in \eqref{eq:cost_function}, and ELLIPSDF (coarse+fine) means the bi-level SDF residuals are used. For each optimized object, we calculate the fitting rate and then average across all instances. In Tab.~\ref{tab:fitting_performance}, we show the number of instances and average fitting rates for 4 object classes.
ELLIPSDF (coarse+fine) achieves better results than ELLIPSDF (fine) across all classes, demonstrating an average 3\% boost of fitting rate with the assistance of coarse model, reaching nearly 90\% accuracy. The results indicate the effectiveness of the coarse level error function for improving the scale estimation. 


\begin{table}[t]
\centering
\caption{Comparison of 3D detection results on ScanNet \cite{dai2017scannet}.}
\small
    \begin{tabular}{l|c|c|c} 
    \hline
    mAP @ IoU=0.5 & Chair & Table & Display \\
    \hline\hline
    FroDO \cite{runz2020frodo} & $0.32$ & $0.06$ & $0.04$\\
    MOLTR \cite{li2020mo} & $0.39$ & $0.06$ & $0.10$ \\
    ELLIPSDF (fine) & 0.42 & 0.26 & 0.25 \\
    ELLIPSDF (coarse+fine) & \bf{0.43} & \bf{0.27} & \bf{0.31} \\
    \hline
    \end{tabular}
    \vspace{0.2cm}
    \label{tab:detection_moltr}
\end{table}

{\vspace{1ex}\bf \noindent Evaluation on 3D IoU: }%
For a quantitative evaluation on pose estimation, our approach is compared with FroDO~\cite{runz2020frodo} and MOLTR~\cite{li2020mo} on ScanNet~\cite{dai2017scannet}. The ground-truth object poses and shapes are from Scan2CAD \cite{avetisyan2019scan2cad}, whereas the estimated 3D bounding box is generated from the estimated point cloud.
The evaluation metric is same as \cite{li2020mo}, i.e. mean Average Precision (mAP), and the IoU threshold is 0.5. The results are shown in Tab. \ref{tab:detection_moltr}. 
First, we compare the bi-level model against the one-level model. From the last two rows in Tab. \ref{tab:detection_moltr}, ELLIPSDF (coarse+fine) is superior than ELLIPSDF (fine) in terms of 3D IoU, and thus demonstrates that the bi-level model is beneficial by providing additional cues to constrain the pose and shape. The improvement is more significant for smaller objects, e.g. the displays. This may be explained by the fact that the initialization error is relatively larger for smaller objects, and thus requires a coarse shape residual to confine its pose. 
Moreover, ELLIPSDF outperforms both FroDO and MOLTR by a large margin for two probably reasons. Firstly, 3D point clouds are used in the observation for ELLIPSDF, while the other two only rely on 2D observations. Secondly, ELLIPSDF computes coarse level SDF residuals using a primitive shape to aid the estimation of pose and shape scale, whereas the other methods use SDF residuals computed from fine shape details. 

%% file: Conclusion.tex
\section{Conclusion}
\label{sec:conclusion}

This work proposes ELLIPSDF, which a novel semantic mapping approach for RGB-D sensors using a compact, shared latent representation for a bi-level object model to achieve joint pose and shape optimizaiton. Evaluation results on large-scale dataset demonstrate the superiority of ELLIPSDF compared with other approaches.  
A future research direction is to integrate ELLIPSDF into the pose graph optimization for key-frame based SLAM.

%% file: Supplemental.tex
\section*{Supplementary Material}

\subsection*{Trained Object Models}

This section provides additional visualizations for the trained object models. Training loss for the chair category is visualize in Fig.~\ref{fig:training_loss_chair}, which shows the loss is decreasing and stabilizes around 40,000 epochs. 

Fig.~\ref{fig:trained_model_chair} visualizes the rendering results for some chairs in the training set. It shows that the scale of the primitive-based representation varies proportionally with the high-resolution representation. 

\begin{figure}[thp!]
    \centering
    \includegraphics[width=\linewidth]{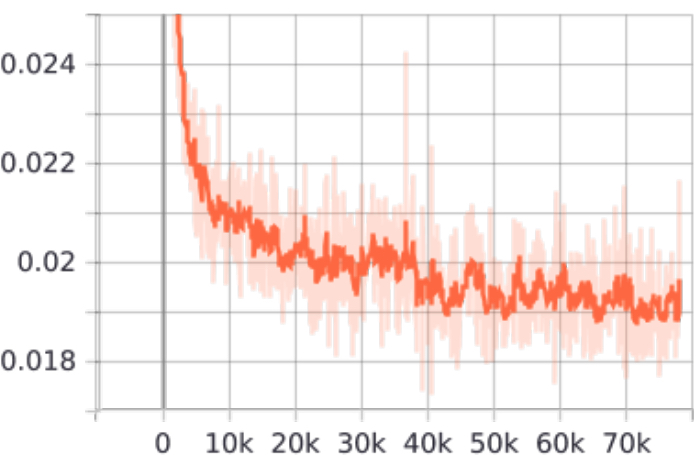}
    \caption{Visualization of the training loss for chairs.}
    \label{fig:training_loss_chair}
\end{figure}

Fig.~\ref{fig:trained_model_sofa} visualizes the rendering results for sofas in the training set. There is a lack of shape variation since the majority of sofas have similar structure. Nevertheless, the ellispoid for the angle sofa is still different with that of other sofas. 

\begin{figure}[thp!]
    \centering
    \includegraphics[width=\linewidth]{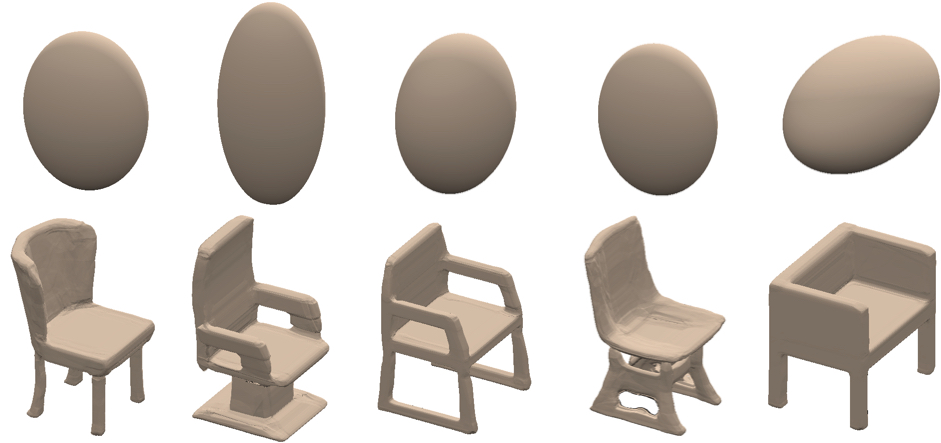}
    \caption{Visualization of the trained object model for chairs. Upper row: coarse ellipsoid shapes regressed from $g_{\bfphi}$ and $\bfz$. Lower row: SDF object model from $f_{\bftheta}$ and $\bfz$.}
    \label{fig:trained_model_chair}
\end{figure}

\begin{figure}[thp!]
    \centering
    \includegraphics[width=\linewidth]{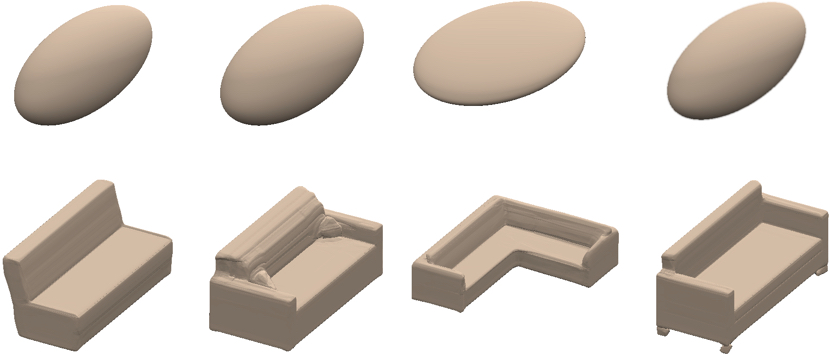}
    \caption{Visualization of the trained object model for sofas. Upper row: coarse ellipsoid shapes regressed from $g_{\bfphi}$ and $\bfz$. Lower row: SDF object model from $f_{\bftheta}$ and $\bfz$.}
    \label{fig:trained_model_sofa}
\end{figure}

Fig.~\ref{fig:trained_model_table} visualizes the rendering results for tables in the training set. Similar to sofas, the variation is limited due to similar table shapes. Nonetheless, the ellipsoid for the rounded table is different from the rest.

\begin{figure}[thp!]
    \centering
    \includegraphics[width=\linewidth]{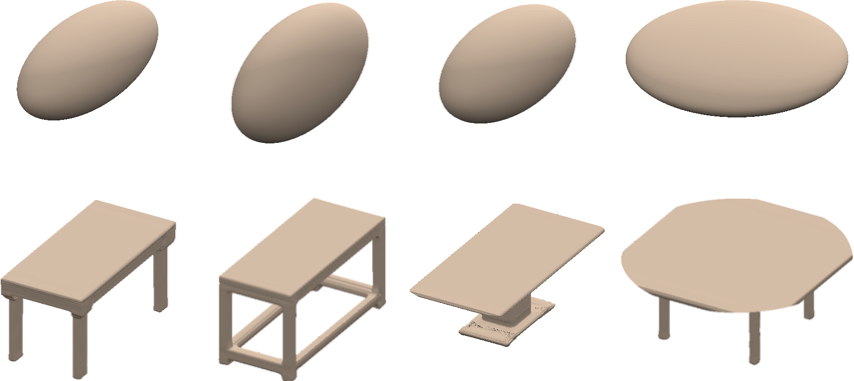}
    \caption{Visualization of the trained object model for tables. Upper row: coarse ellipsoid shapes regressed from $g_{\bfphi}$ and $\bfz$. Lower row: SDF object model from $f_{\bftheta}$ and $\bfz$.}
    \label{fig:trained_model_table}
\end{figure}

Fig.~\ref{fig:trained_model_trashbin} visualizes the rendering results for trashbins in the training set. It could be observed that the ellipsoid shape varies based on the object shape, for instance, the ellipsoid is enlongated for a tall trashbin.

\begin{figure}[thp!]
    \centering
    \includegraphics[width=\linewidth]{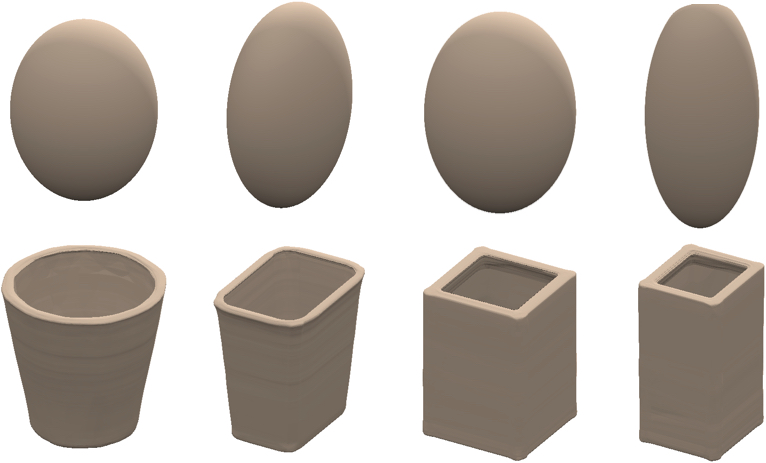}
    \caption{Visualization of the trained object model for trashbins. Upper row: coarse ellipsoid shapes regressed from $g_{\bfphi}$ and $\bfz$. Lower row: SDF object model from $f_{\bftheta}$ and $\bfz$.}
    \label{fig:trained_model_trashbin}
\end{figure}

\begin{figure}[thp!]
    \centering
    \includegraphics[width=\linewidth]{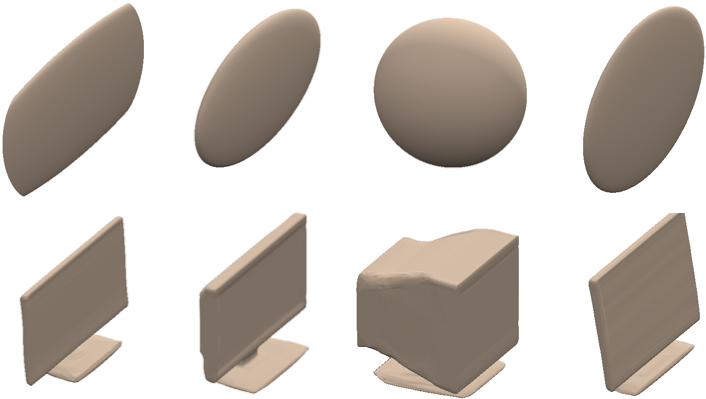}
    \caption{Visualization of the trained object model for displays. Upper row: coarse ellipsoid shapes regressed from $g_{\bfphi}$ and $\bfz$. Lower row: SDF object model from $f_{\bftheta}$ and $\bfz$.}
    \label{fig:trained_model_display}
\end{figure}

\begin{figure}[thp!]
    \centering
    \includegraphics[width=\linewidth]{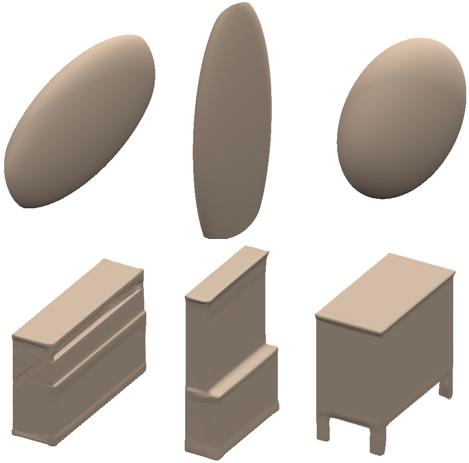}
    \caption{Visualization of the trained object model for cabinets. Upper row: coarse ellipsoid shapes regressed from $g_{\bfphi}$ and $\bfz$. Lower row: SDF object model from $f_{\bftheta}$ and $\bfz$.}
    \label{fig:trained_model_cabinet}
\end{figure}

Fig.~\ref{fig:trained_model_display} visualizes the rendering results for displays in the training set. The ellipsoid is rounded for the thicker display and is very thin for the rest. 

Fig.~\ref{fig:trained_model_cabinet} visualizes the rendering results for cabinets in the training set. The ellipsoid varies according to the different cabinet shapes.

\subsection*{More Qualitative Results on ScanNet}

This section presents more qualitative results on ScanNet~\cite{dai2017scannet}. 
Fig.~\ref{fig:scannet_qualitative_0077_01} shows a reconstruction with table, trashbins, and cabinet. The cabinet and trashbins are reconstructed well, as can be seen from the resulting meshes which resemble the original object shapes. However, the table is poorly reconstructed, since the shape is quite different and the pose is inaccurate. This is because the available observation in the scene for the table is very limited, as can be seen in the segmented mesh, which is insufficient for optimization.

A ScanNet scene with bookshelves and tables are shown in Fig.~\ref{fig:scannet_qualitative_0208_00}, to demonstrate the usefulness of the coarse and fine level residuals. The figure illustrates that the initialized object pose and shape are different from the actual scene, since the two bookshelves in the center are not parallel and are too small compared to the observation. In contrast, the bookshelves become larger after applying the fine level residual, which is more consistent with the observations. The reconstructions are further improved with both the coarse and fine level residuals, where the bookshelves become parallel. Moreover, the bottom bookshelf and the top right table also become thinner, which agrees more with the observation. 
This example clearly shows the effectiveness of the proposed bi-level model for joint object pose and shape optimization.  

\begin{figure}[thp!]
    \centering
    \includegraphics[width=\linewidth]{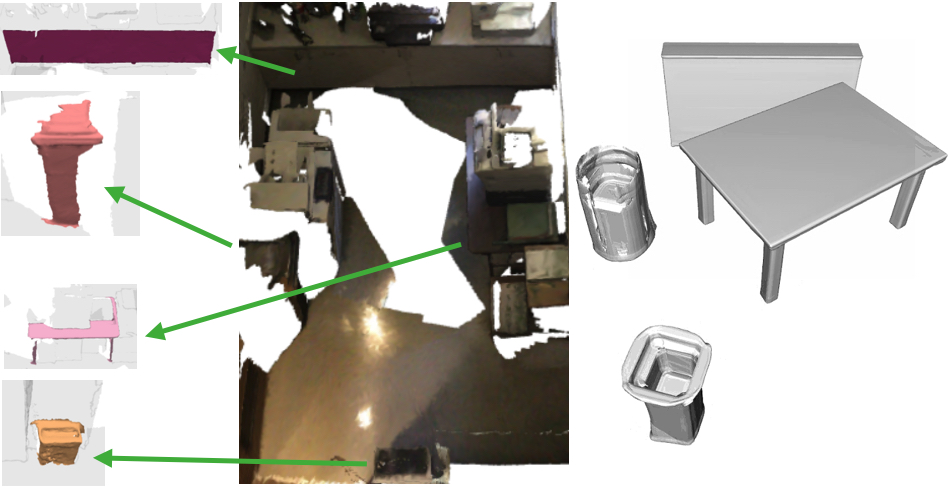}
    \caption{Visualization of the original scene and reconstructed objects for ScanNet scene $0077$. The green arrows point to the segmented mesh of the objects.}
    \label{fig:scannet_qualitative_0077_01}
\end{figure}

\begin{figure*}[thp!]
    \centering
    \includegraphics[width=\linewidth]{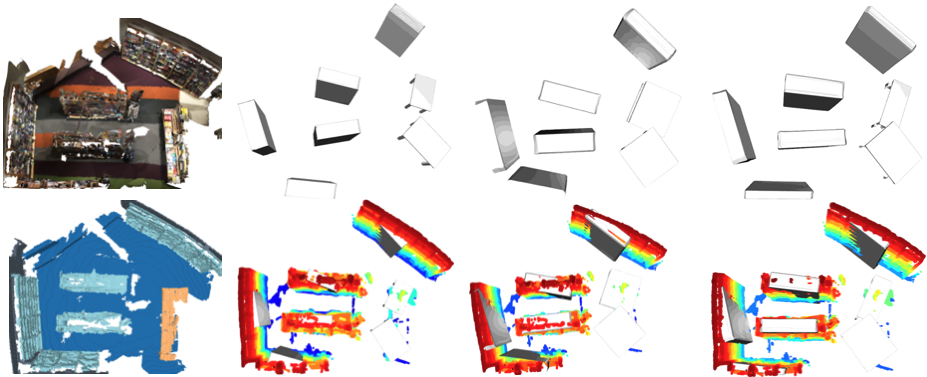}
    \caption{Visualization of the original scene and reconstructed objects for ScanNet scene $0208$. First row from left to right: original scene, reconstruction using initialized pose and mean categorical object shape, reconstruction using optimized pose and shape with fine level residual only, reconstruction using optimized pose and shape with both coarse and fine level residuals. Second row from left to right: original scene with bookshelves and tables highlighted in light blue and beige, the rest are reconstructions overlaid with object point clouds and added pseudo points.}
    \label{fig:scannet_qualitative_0208_00}
\end{figure*}

\subsection*{Pose Estimation Metric}

This section presents the metric used to evaluate the object pose, which follows Scan2CAD~\cite{avetisyan2019scan2cad}. 
We introduce the details on how to decompose a pose $\bfT \in \text{SIM}(3)$ into rotation $\bfq$, translation $\bfp$ and scale $\bfs$ and the error functions for each element separately.
For rotation and scale, $\bfR_s = \bfP\bfT\bfP^\top$:
\begin{equation}
\label{eq:pose_error}
\begin{aligned}
s_1 &= \| \bfR_s\bfe_1 \|_2 \quad 
s_2 = \| \bfR_s\bfe_2 \|_2 \quad 
s_3 = \| \bfR_s\bfe_3 \|_2, \\
\bfR\bfe_1 &= \frac{\bfR_s\bfe_1}{s_1} \quad 
\bfR\bfe_2 = \frac{\bfR_s\bfe_2}{s_2}
\quad 
\bfR\bfe_3 = \frac{\bfR_s\bfe_3}{s_3}. 
\end{aligned}
\end{equation}
Suppose $\boldsymbol{R}=\left\{m_{i j}\right\}, i, j \in[1,2,3]$, we transform it to quaternion $\bfq$ by 
\begin{equation}
\scaleMathLine[0.9]
{
\begin{aligned}
q_{0}=\frac{\sqrt{\operatorname{tr}(R)+1}}{2}, q_{1}=\frac{m_{23}-m_{32}}{4 q_{0}}, q_{2}=\frac{m_{31}-m_{13}}{4 q_{0}}, q_{3}=\frac{m_{12}-m_{21}}{4 q_{0}}. 
\end{aligned}
}
\end{equation}
Suppose the prediction and groundtruth are $\bfq_{pred}, \bfq_{gt}$, we compute the difference by 
\begin{equation}
\begin{aligned}
e_{\text{SO(3)}}(\bfq,\hat{\bfq}) := 2 \arccos (| \bfq_{gt}^\top \bfq_{pred} |). 
\end{aligned}
\end{equation}
Translation is $\bfp = \bfT[1:3, 4]$, and we compare the difference between prediction and groundtruth by 
\begin{equation}
\| \bfp_{pred} - \bfp_{gt} \|_2. 
\end{equation}
For scale percentage error, we compute it by 
\begin{equation}
100\times | \frac{1}{3} \sum_{i=1}^{3} \bar{s}_i - 1 |,
\end{equation}
where $\bar{s}_i = \frac{s_{pred}}{s_{gt}}$ for each of $s_1, s_2, s_3$ recovered from the $\text{SIM}(3)$ matrix. 

\subsection*{Timing}

\begin{table}[tph!]
    \centering
    \caption{ELLIPSDF timing breakdown (sec)}
    \label{tab:time}
    \scalebox{0.78}{
        \begin{tabular}{c|c|c|c|c} 
        \hline
        Init & Latent Code Opt & SIM(3) Opt & SDF Decoding & Meshing \\
        \hline
        0.04 & 0.13 & 0.58 & 1.38 & 2.34 \\
        \hline
        \end{tabular}}
\end{table}

Timing for one instance is provided in Table~\ref{tab:time}. \textit{Init} is the pose initialization in (14) for 100 views. \textit{Latent Code Opt} and \textit{SIM(3) Opt} are a single SGD step with respect to $\delta \bfz$ and $\bfT$ respectively using 10000 points as batch size. \textit{SDF Decoding} and \textit{Meshing} are optional steps that generate SDF predictions over $256^3$ points and apply Marching Cubes to generate a mesh. Our approach does not currently operate in real-time but it is more efficient than existing work. We will investigate how to accelerate the current slow python SIM(3) optimization.